\documentclass[twoside,11pt]{article}

\usepackage{floatrow}
\newfloatcommand{capbtabbox}{table}[][\FBwidth]

\usepackage{amssymb}
\usepackage{amsmath}
\usepackage{graphicx}

\usepackage[utf8]{inputenc} 
\usepackage[T1]{fontenc}    

\usepackage{url}            
\usepackage{booktabs}       
\usepackage{amsfonts}       
\usepackage{nicefrac}       
\usepackage{microtype}      

\usepackage{xcolor}         
\usepackage{pifont}
\usepackage{wrapfig}

\usepackage{mathtools}
\usepackage{tikz}
\usepackage{caption}
\usepackage{subcaption}
\usepackage[ruled, vlined]{algorithm2e} 
\newlength\myindent
\usepackage{lipsum}

\usepackage{jmlr2e}

\usepackage{xargs}

\usepackage{amssymb}
\usepackage{amsmath}
\usepackage{bm}

\DeclareMathOperator*{\argmax}{arg\,max}

\newcommand{\designnet}{\pi_\phi}

\newcommand{\latent}{\theta}

\newcommandx{\histmarg}[1][1=\designnet]{p(h_T| #1)}
\newcommandx{\histlik}[1][1=\designnet]{p(h_T|\latent, #1)}
\newcommandx{\histliki}[2][2=\designnet]{p(h_T|\latent_{#1}, #2)}

\newcommand{\Do}[1]{\text{do}(#1)} 

\newcommand{\R}{\mathbb{R}}
\newcommand{\E}{\mathbb{E}}

\def\rva{{\mathbf{a}}}

\def\rvc{{\mathbf{c}}}

\def\rvm{{\mathbf{m}}}

\def\rvy{{\mathbf{y}}}

\def\rvA{{\mathbf{A}}}
\def\rvC{{\mathbf{C}}}

\newcommand{\unconditionaljoint}{p(\rvy, \rvm^* | \rvC, \rvC^*, \Do{\rvA})}
\newcommand{\unconditionaljointzero}{p(\rvy, \rvm^*_0 | \rvC, \rvC^*, \Do{\rvA})}

\jmlrheading{1}{}{}{}{}{}

\firstpageno{1}

\begin{document}

\title{Efficient Real-world Testing of Causal Decision Making via Bayesian Experimental Design for Contextual Optimisation}

\author{\name Desi R. Ivanova\normalfont{\textsuperscript{\textdagger}} \email desi.ivanova@stats.ox.ac.uk \\
       \addr Department of Statistics, University of Oxford
       \AND
       \name Joel Jennings \email joeljennings@microsoft.com \\
       \addr Microsoft Research Cambridge
       \AND
       \name Cheng Zhang \email chezha@microsoft.com \\
       \addr Microsoft Research Cambridge
       \AND
       \name Adam Foster \email adam.e.foster@microsoft.com \\
       \addr Microsoft Research Cambridge \\
       \normalfont{\textsuperscript{\textdagger}Work done during an internship at Microsoft Research Cambridge}
       }


\maketitle

\begin{abstract}
The real-world testing of decisions made using causal machine learning models is an essential prerequisite for their successful application. 
We focus on evaluating and improving contextual treatment assignment decisions: these are personalised treatments applied to e.g.~customers, each with their own contextual information, with the aim of maximising a reward.
In this paper we introduce a model-agnostic framework for gathering data to evaluate and improve contextual decision making through Bayesian Experimental Design.
Specifically, our method is used for the data-efficient evaluation of  the \emph{regret} of past treatment assignments.
Unlike approaches such as A/B testing, 
our method avoids assigning treatments that are known to be highly sub-optimal, whilst engaging in some exploration to gather pertinent information. 
We achieve this by 
introducing an information-based design objective, which we optimise end-to-end.
Our method applies to discrete and continuous treatments.
Comparing our information-theoretic approach to baselines in several simulation studies demonstrates the superior performance of our proposed approach.


\end{abstract}


\section{Introduction}\label{sec:introduction}

Machine learning models, particularly causal machine learning models, can be used to make real-world decisions \citep{joyce1999foundations,geffner2022deep}. 
We consider the problem of collecting experimental data to evaluate such models and improve decisions made in future.
For example, suppose a company has developed a model describing the relationship between customer engagement and revenue. 
This model will be used to offer personalised promotions to each customer in order to maximise some reward, e.g.~revenue.   
Before deploying it to all of its customers, 
the company may wish to collect experimental data to validate the model.

To formalise this, we consider the contextual optimisation problem in which a treatment or treatments must be assigned to each context to receive a reward.
We focus on the experimental design problem of finding treatments that allow us to evaluate the regret of past actions and to improve decisions made in the future.
Our aim, then, is to choose actions to test the model; we are not directly optimising rewards in the experimental phase.
In line with our real-world scenario, we focus on designing \emph{large batch} experiments, in which we must select a large number of treatments for different contexts before receiving any feedback.

A standard experimental design procedure is random treatment assignment, which corresponds to A/B testing when treatment is binary \citep{fisher1936design,kohavi2017online}.
However, this naive approach has significant drawbacks, as it may involve applying many sub-optimal actions, meaning that the \emph{opportunity cost of experimentation} becomes high and experimental resources are wasted applying treatments that are known \emph{a priori} to be sub-optimal. On the other hand, applying the best treatments based on existing beliefs---a pure exploitation strategy---may also be uninformative for model testing as it fails to account for uncertainty and the need to gather data where uncertainty is high. 

To design an experiment in a more principled way, we turn to the framework of Bayesian experimental design (BED) \citep{lindley1956,chaloner1995,foster2019variational}. In contrast to standard BED, which maximises the information gained about all model parameters, we define the objective for experimental design as the \emph{expected information gained about the maximum reward obtainable for a set of evaluation contexts (max-value EIG)}.
This directly addresses the problem of gaining data to evaluate the regret of past decisions, because, given a past outcome, the unknown component of regret is \emph{how much better the reward could have been}. The formulation is also relevant to the problem of gathering data (exploration) to obtain better rewards in later interactions (exploitation). 

As a model-based framework, BED requires the specification of a Bayesian model \citep{gelman2013bayesian} relating contexts, treatments and rewards. 
Since the model makes predictions about interventional distributions, causality plays a central role in selecting a correct model, particularly when the model is pre-trained with observational data.
One approach is to rely on assumptions about the causal graph \citep{pearl2009causal,sharma2021dowhy}.
Alternatively, the model can incorporate Bayesian uncertainty about the causal structure \citep{heckerman1999bayesian,annadani2021variational,geffner2022deep}.
In this paper, we propose a model-agnostic approach to experimental design that accommodates this full range of Bayesian models; we only require \emph{differentiability} of samples from the model.

To find the optimal experimental design for a given Bayesian model, we first show that the InfoNCE bound \citep{oord2018representation} can be used to estimate the max-value EIG objective.
We then optimise the testing treatment assignment (experimental design) by gradients \citep{foster2020unified,kleinegesse2020minebed}.
Our method is applicable to multiple treatments, both continuous and discrete, using carefully tuned Gumbel--Softmax \citep{maddison2016concrete,jang2016categorical} for the latter case.
We show experimentally that our approach is effective on a number of synthetic models, learning intuitively correct large batch designs that outperform standard baselines by significant margins.

\section{Problem set-up} 
We are interested in the efficient evaluation of the regret of past treatment assignment decisions.
We formalise this by considering a contextual optimisation problem in which treatment, $\rva$, is applied in context, $\rvc$, to receive reward $y$.
For example, in the customer engagement scenario, $\rvc$ would be the information we have about the customer that cannot be directly acted upon, such as age, income and historical data of past interactions, $\rva$ would denote a particular choice of promotions that we offer to the customer, and the reward $y$ would be the future revenue generated by the customer, net of promotion costs. 


\paragraph{Expected regret}  We are interested in gathering data to evaluate the \emph{expected regret} of the past action $\rva$ performed in a context of interest $\rvc^*$. This is defined as
\begin{equation}
    R(\rva, \rvc^*) = m(\rvc^*) -  \E[y|\Do{\rva},\rvc^*], \label{eq:regret}
\end{equation}
where $p(\cdot|\Do{\rva},\rvc^*)$ denotes the interventional distribution when action $\rva$ is applied to context $\rvc^*$ and $m(\rvc^*) = \max_{\rva^{\prime}} \E[y | \Do{\rva^{\prime}},\rvc^*]$ denotes the best obtainable outcome for context $\rvc^*$.
Suppose that $\rva$ has been performed historically to $\rvc^*$, and we wish evaluate the regret of this treatment. 
The primary obstacle to evaluating this regret in the real world is that $m(\rvc^*)$ is not known, i.e.~we do not know how much better the outcome could have been if a different treatment had been applied. 
We therefore wish to obtain experimental data, possibly from different contexts, that will help us to efficiently infer $m(\rvc^*)$.
Gathering data to learn $m(\rvc^*)$ also aids in choosing the best treatment in \emph{future} interactions with $\rvc^*$.



\paragraph{Bayesian model} In general, we also do not know the true interventional distribution $p(y | \Do{\rva}, \rvc)$.
We therefore model the relationship between $\rva,\rvc$ and $y$ is a Bayesian manner, by introducing a Bayesian parameter $\psi$ with a prior $p(\psi)$.
Different samples of $\psi$ corresponds to different hypotheses $p(y|\text{do}(\rva),\rvc,\psi)$, whilst marginalising over $\psi$ summarises the total Bayesian uncertainty in $y|\text{do}(\rva),\rvc$.
Figure~\ref{fig:stylized_example} illustrates an example with a binary treatment.


\paragraph{The role of causality} 
Causal considerations are central to building a correct model of the interventional distribution $p(y|\Do{\rva},\rvc)$, particularly when the prior $p(\psi)$ is fitted by conditioning on past \emph{observational} data.
By making certain assumptions about the causal graph and the (non-)existence of unobserved confounders, one can reduce the problem to that of learning specific functions \citep{pearl2009causal,sharma2021dowhy}.
These assumptions are often made implicitly in the optimisation literature \citep{bareinboim2015bandits}.
Through suitable assumptions, any Bayesian model that predicts $y$ from $\rva,\rvc$ can be used as a model for $p(y|\Do{\rva},\rvc)$, so a causal treatment of the problem does not restrict the class of models that can be considered.
Alternatively, the Bayesian parameter $\psi$ can incorporate uncertainty in the causal graph \citep{heckerman1999bayesian,annadani2021variational,geffner2022deep}.
Our approach to experimental design is concordant with either of these modelling choices.

\begin{figure}[t]
\vspace{-20pt}
\centering
    \includegraphics[width=0.78\textwidth]{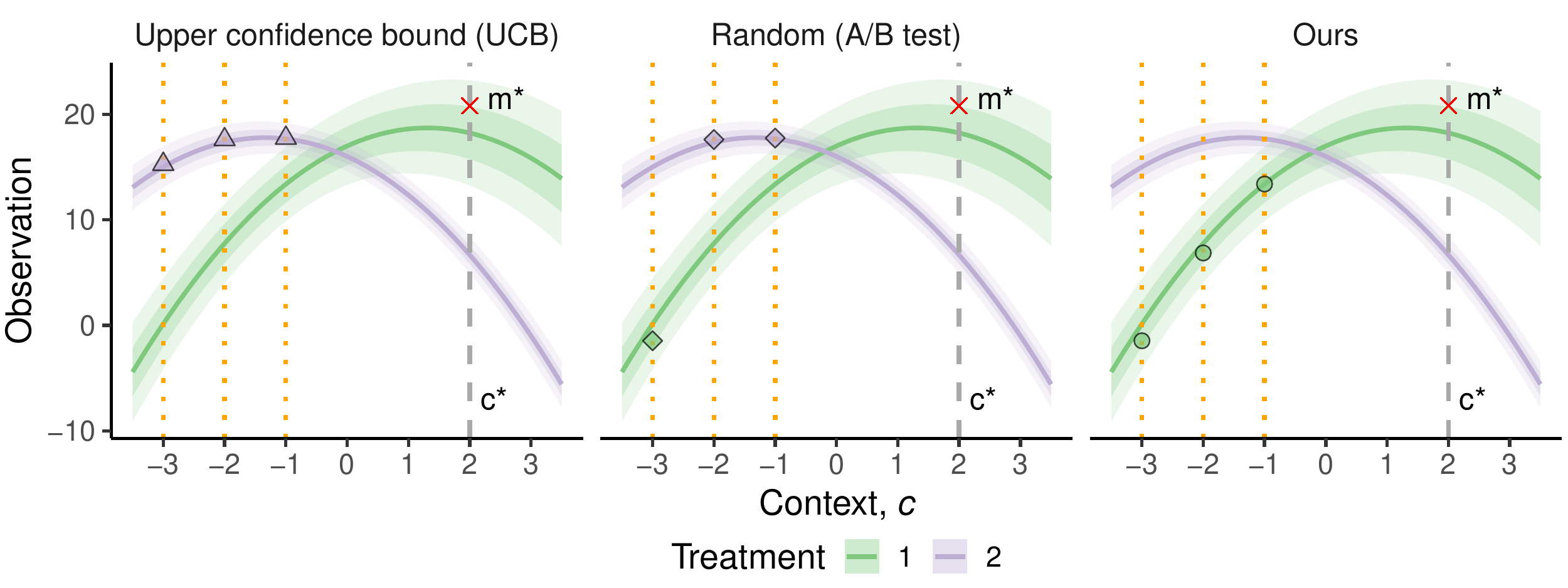}
    \vspace{-10pt}
    \caption{An example with a scalar context and a binary treatment $\rva \in \{1,2\}$.
    The dark (resp. light) shaded area shows our uncertainty about $y|\Do{\rva},\rvc$ arising from uncertainty in $\psi$, measured by one (resp. two) st.~dev.~from the mean.
    We want to estimate $m^*=m(\rvc^*)$, the unknown maximum achievable reward at the evaluation context $\rvc^*$.
    We perform three experiments at $\rvc_1, \rvc_2, \rvc_3$ (dotted orange lines). 
    UCB chooses the reward-maximising treatments in each experimental context, bringing no information about $m^*$, as Treatment 2 is \emph{a priori} suboptimal at $\rvc^*$.
    Random design selects Treatments 1 and 2 with equal probability;
    our method selects treatments whose outcomes will be most informative about~$\rvm^*$ given the prior.
    \vspace{-18pt}
    }
    \label{fig:stylized_example}
\end{figure}
 
\section{Bayesian Experimental Design for Contextual Optimisation}

Whilst standard BED would design experiments to maximise the information gathered about the Bayesian parameter $\psi$, our goal is to efficiently evaluate the regret~\eqref{eq:regret}. This translates to estimating the maximum achievable rewards for a set of $D^*$ evaluation contexts $\rvC^* = \rvc^*_1,\dots,\rvc^*_{D^*}$, denoted $\rvm^* = m(\rvc^*_1),\dots,m(\rvc^*_{D^*})$.
In the customer engagement example, $\rvC^*$ represents a set of customers for which we are aiming to infer the best possible rewards $\rvm^*$.
To learn about $\rvm^*$, we are allowed to perform a batch of $D$ experiments at contexts $\rvC = \rvc_1,\dots,\rvc_D$.
In the customer engagement example, $\rvC$ represents the set of customers that will participate in the real-world test.
The \emph{experimental design} problem is to select the treatments $\rvA = \rva_1,\dots,\rva_D$ to apply to $\rvC$ so that the experimental outcomes $\rvy = y_1,\dots,y_D $ that we get from the real-world test are maximally informative about $\rvm^*$.
More formally, we design $\rvA$ to maximize expected information gain 
between $\rvy$ and $\rvm^*$ 
\begin{align}
    \mathcal{I}(\rvA; \rvC, \rvC^*) %
    &= \E_{\unconditionaljoint}  
    \left[ 
        \log \frac{  p(\rvy |\rvm^*, \Do{\rvA},  \rvC) }{ p(\rvy|\Do{\rvA},  \rvC) }
    \right], \label{eq:mi_objective}
\end{align}
where $p(\rvy|\Do{\rvA},  \rvC) = \E_{p(\psi)} \big[ p(\rvy|\Do{\rvA},  \rvC, \psi) \big]$, $p(\rvm^* | \rvC^*) = \E_{p(\psi)} \big[p(\rvm^* | \rvC^*, \psi)\big]$ and $\unconditionaljoint = \E_{p(\psi)} \big[ p(\rvm^* | \rvC^*, \psi) p(\rvy | \rvC, \Do{\rvA}, \psi) \big]$. 
We refer to $\mathcal{I}(\rvA; \rvC, \rvC^*)$ as the max-value EIG.
The objective is doubly intractable \citep{rainforth2018nesting} meaning that its optimisation, or even evaluation with a fixed $\rvA$, is a major challenge. The likelihoods in the expectation are \emph{implicit} (not available analytically) making the problem more difficult. 

\subsection{Training stage: Objective estimation and optimization}\label{sec:training_stage}

Recently, gradient-based approaches for optimising EIG 
have been developed for experimental design for parameter learning \citep{foster2020unified, kleinegesse2020minebed} as opposed to contextual optimisation, which is the focus of this work. 
Inspired by likelihood-free mutual information estimators, we show the InfoNCE bound \citep{oord2018representation} can be adapted to the contextual optimisation setting we consider here. Concretely,
\begin{equation}
    \mathcal{L}(\rvA, U; \rvC, \rvC^*, L) = \E_{p(\psi)p(\rvy, \rvm^*_0 | \Do{\rvA},  \rvC, \rvC^*, \psi)p(\rvm^*_{1:L}| \rvC^*)} 
    \left[ 
        \log \frac{  \exp(U(\rvy , \rvm^*_0))}{ \frac{1}{L+1} \sum_\ell \exp(U(\rvy , \rvm^*_\ell)) }
    \right] \label{eq:infonce_bound}
\end{equation}
is a lower bound on~$\mathcal{I}(\rvA;\rvC,\rvC^*)$ for any $L \geq 1$ and any critic function $U$. We provide a proof of this claim in Appendix~\ref{sec:appendix_proofs}. To learn a suitable critic $U$, we use a neural network with trainable parameters $\phi$.
We then optimise the lower bound $\mathcal{L}(\rvA, U; \rvC, \rvC^*, L)$ simultaneously with respect to $\phi$ and $\rvA$, thereby improving the tightness of the bound and optimising the experimental design together.
When the actions $\rvA$ are continuous, we require that a sample $y$ of $p(y|\Do{\rva},\rvc,\psi)$ is differentiable with respect to $\rva$. 
In this case, a pathwise gradient estimator \citep{mohamed2020monte} for $\nabla_{\rvA, \phi}\mathcal{L}$ can be readily computed and we apply standard stochastic gradient ascent (SGA) \citep{robbins1951stochastic} to optimise our objective~\eqref{eq:infonce_bound}.




\paragraph{Discrete action space}
Previous gradient-based BED work has focused on fully differentiable models. 
Here we propose a practical way to deal with discrete designs. 
Rather than learning the treatments $\rvA$ directly, we introduce a stochastic treatment policy $\pi_{\bm{\alpha}}$ with trainable parameters $\bm{\alpha}$ representing the probabilities of selecting each treatment. 
During training we use a Gumbel--Softmax \citep{maddison2016concrete, jang2016categorical} relaxed form of $\rvA$ and update $\bm{\alpha}$ by gradients.
Further detail is provided in Appendix~\ref{sec:appendix_proofs}.

\subsection{Deployment stage: Real-world testing}\label{sec:deployment}
Once we optimise the objective~\eqref{eq:infonce_bound}, we perform the batch of experiments $\rvA$ to obtain real-world outcomes $\rvy$ and estimate $p(\psi|\mathcal{D})$---the posterior of our Bayesian model, given the experimental data $\mathcal{D}=(\rvy, \rvA, \rvC)$.
We then calculate the posterior $p(\rvm^*|\rvC^*,\mathcal{D})$, and this it to estimate the regret~\eqref{eq:regret} of a set of past treatments.
Importantly, through our design of $\rvA$, the data contained in $\mathcal{D}$ should lead to the \emph{most accurate} estimate of $\rvm^*$, which translates to an accurate estimate of the regret and an accurate evaluation metric for the original decisions.
%
Algorithm \ref{algo:method} in the Appendix gives a complete description of our method. 

\section{Related Work}
The most closely related objective to our max-value EIG \eqref{eq:mi_objective} is Max-value Entropy Search \citep[MES,][]{wang2017max}. It was proposed as a computationally efficient alternative to Entropy Search \citep[ES,][]{hennig2012entropy} and Predictive Entropy Search \citep[PES, ][]{hernandez2014predictive}. Our objective 
differs from MES in two essential ways: firstly, MES is applicable to Gaussian processes (GPs), whilst our objective is model-agnostic; secondly, we focus on \emph{contextual} optimisation, rather than finding a single maximiser. For a more comprehensive discussion of related work, see Appendix~\ref{sec:app:relwork}.

\section{Empirical Evaluation}

\begin{wrapfigure}[7]{r}{0.29\textwidth}
\vspace{-25pt}

\tikzset{every picture/.style={line width=0.75pt}} 

\begin{tikzpicture}[x=0.75pt,y=0.75pt,yscale=-1,xscale=1]

\draw   (123.38,152.05) .. controls (129.18,146.14) and (138.68,146.06) .. (144.59,151.86) .. controls (150.5,157.67) and (150.59,167.16) .. (144.78,173.07) .. controls (138.98,178.98) and (129.48,179.07) .. (123.57,173.27) .. controls (117.66,167.46) and (117.57,157.96) .. (123.38,152.05) -- cycle ;
\draw   (76.09,200.92) .. controls (81.9,195.01) and (91.39,194.92) .. (97.3,200.73) .. controls (103.21,206.53) and (103.3,216.03) .. (97.5,221.94) .. controls (91.69,227.85) and (82.19,227.94) .. (76.28,222.13) .. controls (70.37,216.33) and (70.29,206.83) .. (76.09,200.92) -- cycle ;
\draw   (26.86,152.58) .. controls (32.66,146.67) and (42.16,146.58) .. (48.07,152.38) .. controls (53.98,158.19) and (54.07,167.69) .. (48.26,173.6) .. controls (42.46,179.51) and (32.96,179.59) .. (27.05,173.79) .. controls (21.14,167.98) and (21.06,158.49) .. (26.86,152.58) -- cycle ;
\draw    (116.25,162.24) -- (52.76,162.24) ;
\draw [shift={(119.25,162.24)}, rotate = 180] [fill={rgb, 255:red, 0; green, 0; blue, 0 }  ][line width=0.08]  [draw opacity=0] (7.14,-3.43) -- (0,0) -- (7.14,3.43) -- cycle    ;
\draw    (48.26,173.6) -- (73.95,198.82) ;
\draw [shift={(76.09,200.92)}, rotate = 224.48] [fill={rgb, 255:red, 0; green, 0; blue, 0 }  ][line width=0.08]  [draw opacity=0] (7.14,-3.43) -- (0,0) -- (7.14,3.43) -- cycle    ;
\draw    (123.57,173.27) -- (99.38,198.56) ;
\draw [shift={(97.3,200.73)}, rotate = 313.73] [fill={rgb, 255:red, 0; green, 0; blue, 0 }  ][line width=0.08]  [draw opacity=0] (7.14,-3.43) -- (0,0) -- (7.14,3.43) -- cycle    ;
\draw (37.56,163.09) node    {$\rvc$};
\draw (134.08,162.56) node    {$\rva$};
\draw (86.79,211.43) node    {$y$};

\end{tikzpicture}
\caption{
Causal graph.
}
\label{fig:scr_graph}
\end{wrapfigure}
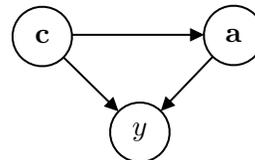
We test the efficacy of our method on several synthetic experiments. 
In all synthetic experiments, we assume the causal graph is as shown in Figure~\ref{fig:scr_graph}, but note that our method is not restricted to this graph structure. 
We compare our method against random treatment assignment and the Upper Confidence Bound (UCB) \citep{auer2002using} algorithm. 
Our key evaluation metrics are \textbf{EIG} and \textbf{accuracy} of inferring $\rvm^*$. For the former we evaluate~\eqref{eq:infonce_bound} with the trained critic and learnt designs. For the latter, we sample a realisation $\tilde{\psi}$ and $\tilde{\rvm}^*$ from the prior, treating it as the ground truth environment in the Deployment stage of our framework (\S~\ref{sec:deployment}). 
We report the mean squared error (MSE) between our estimate $\rvm^*$ and the ground truth $\tilde{\rvm}^*$. We report three additional metrics---average regret, MSE of estimating $\tilde{\psi}$, and deviation from the true optimal treatment. Full details about the models, training, and further results are given in Appendix~\ref{sec:appendix_experiments}.

\begin{figure}[t]
  \vspace{-15pt}
  \centering
  \includegraphics[width=0.92\textwidth]{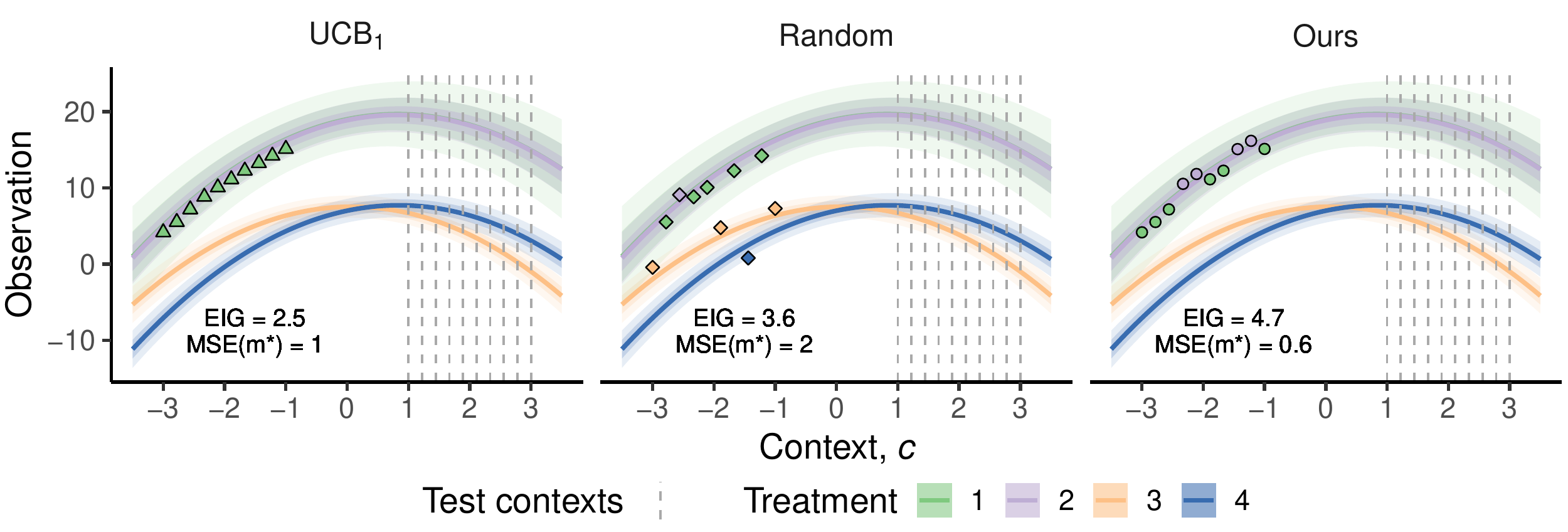}
\vspace{-15pt}
\caption{Discrete treatments: optimal designs and key metrics. More details in Appendix~\ref{sec:appendix_experiments}. 
\vspace{-20pt}
}
\label{fig:exp_discrete_ab}
\end{figure}
\paragraph{Discrete treatments}
We consider a model with four treatment options, two of which are \emph{a priori} sub-optimal; the other two have the same mean but Treatment 1 has higher variance than Treatment 2 (Figure~\ref{fig:exp_discrete_ab}).
Intuitively, the optimal strategy would be to A/B test only the top 2 treatments, which is exactly what our method has learnt. The random method is wasteful as it queries the sub-optimal treatments, while $\text{UCB}_1$ only ever queries Treatment 1. As a result, our method is much more accurate at estimating the true $\rvm^*$.

\paragraph{Continuous treatments}
We test our method on a continuous treatment problem, where we design a batch of 40 experiments to learn about the max-values at 39 evaluation contexts. As Table~\ref{tab:cts_40d_main} shows, our method outperforms the baselines on all metrics. Although our method is optimised to accurately estimate $\tilde{\rvm}^*$, it is also significantly better at estimating $\tilde{\psi}$, which in turn translates to choosing better treatments in the future, as the lower $\text{MSE}(\rvA)$ indicates.
\begin{table}[t]
\resizebox{0.95\textwidth}{!}{
\begin{tabular}{lccccc}
Method             & EIG estimate      & MSE$(\rvm^*)$       & MSE$(\psi)$       & MSE$(\rvA)$       & Regret    \\
\toprule
Random$_{0.2}$     & 5.407 $\pm$ 0.003 & 0.0041 $\pm$ 0.0001 & 0.0114 $\pm$ 0.0003 & 0.544 $\pm$ 0.023 & 0.091 $\pm$ 0.002 \\
Random$_{1.0}$     & 5.798 $\pm$ 0.004 & 0.0024 $\pm$ 0.0002 & 0.0053 $\pm$ 0.0002 & 0.272 $\pm$ 0.018 & 0.060 $\pm$ 0.002  \\
Random$_{2.0}$     & 4.960 $\pm$ 0.004  & 0.0042 $\pm$ 0.0002 & 0.0102 $\pm$ 0.0003 & 0.450 $\pm$ 0.021  & 0.090 $\pm$ 0.002  \\
$\text{UCB}_{0}$ & 5.774 $\pm$ 0.003 & 0.0069 $\pm$ 0.0005 & 0.0154 $\pm$ 0.0007 & 0.747 $\pm$ 0.055 & 0.082 $\pm$ 0.002 \\
$\text{UCB}_{1}$ & 5.876 $\pm$ 0.003 & 0.0030 $\pm$ 0.0002  & 0.0102 $\pm$ 0.0004 & 0.338 $\pm$ 0.024 & 0.067 $\pm$ 0.002 \\
$\text{UCB}_{2}$ & 5.780 $\pm$ 0.004  & 0.0031 $\pm$ 0.0002 & 0.0093 $\pm$ 0.0003 & 0.378 $\pm$ 0.031 & 0.069 $\pm$ 0.002 \\
\midrule
\textbf{Ours}               & \textbf{6.527} $\pm$ \textbf{0.003} &\textbf{ 0.0014} $\pm$ \textbf{0.0001} & \textbf{0.0039} $\pm$ \textbf{0.0002} & \textbf{0.143} $\pm$ \textbf{0.018} &\textbf{ 0.044} $\pm$\textbf{ 0.001}
\vspace{-15pt}
\end{tabular}
}
\caption{Continuous treatments: evaluation of 40D design. More details in Appendix~\ref{sec:appendix_experiments}.} 
\label{tab:cts_40d_main}
\end{table}

\vspace{-10pt}
\section{Discussion and Conclusion}
By combining ideas from Bayesian experimental design, contextual bandits and Bayesian optimisation, we proposed a method for real-world model testing by regret evaluation of past context-dependent treatment assignments. 
Our method also applies to gathering data to improve future decision making, in fact, we saw that gaining information for regret evaluation and for model improvement are really two sides of the same coin.
To compute large batch designs efficiently, we cast the problem in terms of implicit likelihood BED, resulting in a model agnostic approach that is able to handle both discrete and continuous designs.

Future work may introduce the ability to handle constraints, and explore the scaling of our method to high dimensional treatments and contexts, both continuous and discrete.
Our method could also be extended to apply to sequential, adaptive experimentation. 


\newpage
\vskip 0.2in
\bibliography{refs}

\begin{thebibliography}{37}
\providecommand{\natexlab}[1]{#1}
\providecommand{\url}[1]{\texttt{#1}}
\expandafter\ifx\csname urlstyle\endcsname\relax
  \providecommand{\doi}[1]{doi: #1}\else
  \providecommand{\doi}{doi: \begingroup \urlstyle{rm}\Url}\fi

\bibitem[Agrawal and Goyal(2013)]{agrawal2013thompson}
Shipra Agrawal and Navin Goyal.
\newblock Thompson sampling for contextual bandits with linear payoffs.
\newblock In \emph{International conference on machine learning}, pages
  127--135. PMLR, 2013.

\bibitem[Annadani et~al.(2021)Annadani, Rothfuss, Lacoste, Scherrer, Goyal,
  Bengio, and Bauer]{annadani2021variational}
Yashas Annadani, Jonas Rothfuss, Alexandre Lacoste, Nino Scherrer, Anirudh
  Goyal, Yoshua Bengio, and Stefan Bauer.
\newblock Variational causal networks: Approximate bayesian inference over
  causal structures.
\newblock \emph{arXiv preprint arXiv:2106.07635}, 2021.

\bibitem[Auer(2002)]{auer2002using}
Peter Auer.
\newblock Using confidence bounds for exploitation-exploration trade-offs.
\newblock \emph{Journal of Machine Learning Research}, 3\penalty0
  (Nov):\penalty0 397--422, 2002.

\bibitem[Bareinboim et~al.(2015)Bareinboim, Forney, and
  Pearl]{bareinboim2015bandits}
Elias Bareinboim, Andrew Forney, and Judea Pearl.
\newblock Bandits with unobserved confounders: A causal approach.
\newblock \emph{Advances in Neural Information Processing Systems}, 28, 2015.

\bibitem[Bingham et~al.(2018)Bingham, Chen, Jankowiak, Obermeyer, Pradhan,
  Karaletsos, Singh, Szerlip, Horsfall, and Goodman]{pyro}
Eli Bingham, Jonathan~P Chen, Martin Jankowiak, Fritz Obermeyer, Neeraj
  Pradhan, Theofanis Karaletsos, Rohit Singh, Paul Szerlip, Paul Horsfall, and
  Noah~D Goodman.
\newblock Pyro: Deep universal probabilistic programming.
\newblock \emph{Journal of Machine Learning Research}, 2018.

\bibitem[Chaloner and Verdinelli(1995)]{chaloner1995}
Kathryn Chaloner and Isabella Verdinelli.
\newblock {Bayesian experimental design: A review}.
\newblock \emph{Statistical Science}, pages 273--304, 1995.

\bibitem[Char et~al.(2019)Char, Chung, Neiswanger, Kandasamy, Nelson, Boyer,
  Kolemen, and Schneider]{char2019offline}
Ian Char, Youngseog Chung, Willie Neiswanger, Kirthevasan Kandasamy, Andrew~O
  Nelson, Mark Boyer, Egemen Kolemen, and Jeff Schneider.
\newblock Offline contextual bayesian optimization.
\newblock \emph{Advances in Neural Information Processing Systems}, 32, 2019.

\bibitem[Chu et~al.(2011)Chu, Li, Reyzin, and Schapire]{chu2011contextual}
Wei Chu, Lihong Li, Lev Reyzin, and Robert Schapire.
\newblock Contextual bandits with linear payoff functions.
\newblock In \emph{Proceedings of the Fourteenth International Conference on
  Artificial Intelligence and Statistics}, pages 208--214. JMLR Workshop and
  Conference Proceedings, 2011.

\bibitem[Fisher(1936)]{fisher1936design}
Ronald~Aylmer Fisher.
\newblock Design of experiments.
\newblock \emph{British Medical Journal}, 1\penalty0 (3923):\penalty0 554,
  1936.

\bibitem[Foster et~al.(2019)Foster, Jankowiak, Bingham, Horsfall, Teh,
  Rainforth, and Goodman]{foster2019variational}
Adam Foster, Martin Jankowiak, Elias Bingham, Paul Horsfall, Yee~Whye Teh,
  Thomas Rainforth, and Noah Goodman.
\newblock {Variational Bayesian Optimal Experimental Design}.
\newblock In \emph{Advances in Neural Information Processing Systems 32}, pages
  14036--14047. Curran Associates, Inc., 2019.

\bibitem[Foster et~al.(2020)Foster, Jankowiak, O’Meara, Teh, and
  Rainforth]{foster2020unified}
Adam Foster, Martin Jankowiak, Matthew O’Meara, Yee~Whye Teh, and Tom
  Rainforth.
\newblock A unified stochastic gradient approach to designing bayesian-optimal
  experiments.
\newblock In \emph{International Conference on Artificial Intelligence and
  Statistics}, pages 2959--2969. PMLR, 2020.

\bibitem[Geffner et~al.(2022)Geffner, Antoran, Foster, Gong, Ma, Kiciman,
  Sharma, Lamb, Kukla, Pawlowski, et~al.]{geffner2022deep}
Tomas Geffner, Javier Antoran, Adam Foster, Wenbo Gong, Chao Ma, Emre Kiciman,
  Amit Sharma, Angus Lamb, Martin Kukla, Nick Pawlowski, et~al.
\newblock Deep end-to-end causal inference.
\newblock \emph{arXiv preprint arXiv:2202.02195}, 2022.

\bibitem[Gelman et~al.(2013)Gelman, Carlin, Stern, Dunson, Vehtari, and
  Rubin]{gelman2013bayesian}
Andrew Gelman, John~B Carlin, Hal~S Stern, David~B Dunson, Aki Vehtari, and
  Donald~B Rubin.
\newblock \emph{Bayesian data analysis}.
\newblock Chapman and Hall/CRC, 2013.

\bibitem[Ginsbourger et~al.(2014)Ginsbourger, Baccou, Chevalier, Perales,
  Garland, and Monerie]{ginsbourger2014bayesian}
David Ginsbourger, Jean Baccou, Cl{\'e}ment Chevalier, Fr{\'e}d{\'e}ric
  Perales, Nicolas Garland, and Yann Monerie.
\newblock Bayesian adaptive reconstruction of profile optima and optimizers.
\newblock \emph{SIAM/ASA Journal on Uncertainty Quantification}, 2\penalty0
  (1):\penalty0 490--510, 2014.

\bibitem[Han et~al.(2020)Han, Zhou, Zhou, Blanchet, Glynn, and
  Ye]{han2020sequential}
Yanjun Han, Zhengqing Zhou, Zhengyuan Zhou, Jose Blanchet, Peter~W Glynn, and
  Yinyu Ye.
\newblock Sequential batch learning in finite-action linear contextual bandits.
\newblock \emph{arXiv preprint arXiv:2004.06321}, 2020.

\bibitem[Heckerman et~al.(1999)Heckerman, Meek, and
  Cooper]{heckerman1999bayesian}
David Heckerman, Christopher Meek, and Gregory Cooper.
\newblock A bayesian approach to causal discovery.
\newblock \emph{Computation, causation, and discovery}, 19:\penalty0 141--166,
  1999.

\bibitem[Hennig and Schuler(2012)]{hennig2012entropy}
Philipp Hennig and Christian~J Schuler.
\newblock Entropy search for information-efficient global optimization.
\newblock \emph{Journal of Machine Learning Research}, 13\penalty0
  (Jun):\penalty0 1809--1837, 2012.

\bibitem[Hern{\'a}ndez-Lobato et~al.(2014)Hern{\'a}ndez-Lobato, Hoffman, and
  Ghahramani]{hernandez2014predictive}
Jos{\'e}~Miguel Hern{\'a}ndez-Lobato, Matthew~W Hoffman, and Zoubin Ghahramani.
\newblock Predictive entropy search for efficient global optimization of
  black-box functions.
\newblock In \emph{Advances in neural information processing systems}, pages
  918--926, 2014.

\bibitem[Ivanova et~al.(2021)Ivanova, Foster, Kleinegesse, Gutmann, and
  Rainforth]{ivanova2021implicit}
Desi~R Ivanova, Adam Foster, Steven Kleinegesse, Michael Gutmann, and Tom
  Rainforth.
\newblock {Implicit Deep Adaptive Design: Policy–Based Experimental Design
  without Likelihoods}.
\newblock In \emph{Advances in Neural Information Processing Systems},
  volume~34, pages 25785--25798. Curran Associates, Inc., 2021.
\newblock URL
  \url{https://proceedings.neurips.cc/paper/2021/file/d811406316b669ad3d370d78b51b1d2e-Paper.pdf}.

\bibitem[Jang et~al.(2016)Jang, Gu, and Poole]{jang2016categorical}
Eric Jang, Shixiang Gu, and Ben Poole.
\newblock Categorical reparameterization with gumbel-softmax.
\newblock \emph{arXiv preprint arXiv:1611.01144}, 2016.

\bibitem[Joyce(1999)]{joyce1999foundations}
James~M Joyce.
\newblock \emph{The foundations of causal decision theory}.
\newblock Cambridge University Press, 1999.

\bibitem[Kingma and Ba(2014)]{kingma2014adam}
Diederik~P Kingma and Jimmy Ba.
\newblock Adam: {A} method for stochastic optimization.
\newblock \emph{arXiv preprint arXiv:1412.6980}, 2014.

\bibitem[Kleinegesse and Gutmann(2020)]{kleinegesse2020minebed}
Steven Kleinegesse and Michael Gutmann.
\newblock {B}ayesian experimental design for implicit models by mutual
  information neural estimation.
\newblock In \emph{Proceedings of the 37th International Conference on Machine
  Learning}, Proceedings of Machine Learning Research, pages 5316--5326. PMLR,
  2020.

\bibitem[Kohavi and Longbotham(2017)]{kohavi2017online}
Ron Kohavi and Roger Longbotham.
\newblock Online controlled experiments and a/b testing.
\newblock \emph{Encyclopedia of machine learning and data mining}, 7\penalty0
  (8):\penalty0 922--929, 2017.

\bibitem[Krause and Ong(2011)]{krause2011contextual}
Andreas Krause and Cheng Ong.
\newblock Contextual gaussian process bandit optimization.
\newblock \emph{Advances in neural information processing systems}, 24, 2011.

\bibitem[Lindley(1956)]{lindley1956}
Dennis~V Lindley.
\newblock On a measure of the information provided by an experiment.
\newblock \emph{The Annals of Mathematical Statistics}, pages 986--1005, 1956.

\bibitem[Maddison et~al.(2016)Maddison, Mnih, and Teh]{maddison2016concrete}
Chris~J Maddison, Andriy Mnih, and Yee~Whye Teh.
\newblock The concrete distribution: A continuous relaxation of discrete random
  variables.
\newblock \emph{arXiv preprint arXiv:1611.00712}, 2016.

\bibitem[Mohamed et~al.(2020)Mohamed, Rosca, Figurnov, and
  Mnih]{mohamed2020monte}
Shakir Mohamed, Mihaela Rosca, Michael Figurnov, and Andriy Mnih.
\newblock Monte carlo gradient estimation in machine learning.
\newblock \emph{Journal of Machine Learning Research}, 21\penalty0
  (132):\penalty0 1--62, 2020.

\bibitem[Pearce et~al.(2020)Pearce, Klaise, and Groves]{pearce2020practical}
Michael Pearce, Janis Klaise, and Matthew Groves.
\newblock Practical bayesian optimization of objectives with conditioning
  variables.
\newblock \emph{arXiv preprint arXiv:2002.09996}, 2020.

\bibitem[Pearl(2009)]{pearl2009causal}
Judea Pearl.
\newblock Causal inference in statistics: An overview.
\newblock \emph{Statistics surveys}, 3:\penalty0 96--146, 2009.

\bibitem[Poole et~al.(2019)Poole, Ozair, van~den Oord, Alemi, and
  Tucker]{poole2018variational}
Ben Poole, Sherjil Ozair, A{\"a}ron van~den Oord, Alex Alemi, and George
  Tucker.
\newblock On variational bounds of mutual information.
\newblock In \emph{International Conference on Machine Learning}, pages
  5171--5180, 2019.

\bibitem[Rainforth et~al.(2018)Rainforth, Cornish, Yang, Warrington, and
  Wood]{rainforth2018nesting}
Tom Rainforth, Rob Cornish, Hongseok Yang, Andrew Warrington, and Frank Wood.
\newblock On nesting monte carlo estimators.
\newblock In \emph{International Conference on Machine Learning}, pages
  4267--4276. PMLR, 2018.

\bibitem[Robbins and Monro(1951)]{robbins1951stochastic}
Herbert Robbins and Sutton Monro.
\newblock A stochastic approximation method.
\newblock \emph{The annals of mathematical statistics}, pages 400--407, 1951.

\bibitem[Sharma et~al.(2021)Sharma, Syrgkanis, Zhang, and
  K{\i}c{\i}man]{sharma2021dowhy}
Amit Sharma, Vasilis Syrgkanis, Cheng Zhang, and Emre K{\i}c{\i}man.
\newblock Dowhy: Addressing challenges in expressing and validating causal
  assumptions.
\newblock \emph{arXiv preprint arXiv:2108.13518}, 2021.

\bibitem[van~den Oord et~al.(2018)van~den Oord, Li, and
  Vinyals]{oord2018representation}
A{\"a}ron van~den Oord, Yazhe Li, and Oriol Vinyals.
\newblock Representation learning with contrastive predictive coding.
\newblock \emph{arXiv preprint arXiv:1807.03748}, 2018.

\bibitem[Wang and Jegelka(2017)]{wang2017max}
Zi~Wang and Stefanie Jegelka.
\newblock Max-value entropy search for efficient bayesian optimization.
\newblock In \emph{International Conference on Machine Learning}, pages
  3627--3635. PMLR, 2017.

\bibitem[Zanette et~al.(2021)Zanette, Dong, Lee, and
  Brunskill]{zanette2021design}
Andrea Zanette, Kefan Dong, Jonathan~N Lee, and Emma Brunskill.
\newblock Design of experiments for stochastic contextual linear bandits.
\newblock \emph{Advances in Neural Information Processing Systems},
  34:\penalty0 22720--22731, 2021.

\end{thebibliography}

\newpage
\appendix
\section{Method}\label{sec:appendix_proofs}

\begin{algorithm}[t]
\SetAlgoNoLine
\SetKwInput{Input}{Input}
\SetKwInput{Output}{Output}
\Input{Bayesian simulator $p(y, m^* | \Do{\rva},\rvc, \rvc^*, \psi )$, 
batch size $D$ of experimental contexts $\rvC$, 
batch size $D^*$ of evaluation contexts $\rvC^*$,
initial $\rvA$, initial $U_\phi$,
number of contrastive samples $L \geq 1$}
\Output{ Optimal treatments $\rvA$ for experimental contexts $\rvC$ to be tested}

\textbf{Training stage:} 

\Indp

\While{\textnormal{Computational training budget not exceeded}}{
Sample $\psi \sim p(\psi)$

Sample $\rvy \sim  p(\rvy | \rvC, \Do{\rvA}, \psi) $ and $\rvm^* \sim p(\rvm^* | \rvC^*, \psi)$
    
Compute $\mathcal{L} (\rvA, U_\phi; L)$ \eqref{eq:infonce_bound} and update the parameters $(\rvA, \phi)$ using a SGA
}
\Indm
\textbf{Deployment stage:}

\Indp
Run a batch of experiments $\rvA$ to obtain real-world observations $\rvy$ and estimate a posterior $p(\psi |\mathcal{D})$, $\mathcal{D} = \{\rvy, \rvA, \rvC\}$ .

Use the updated model to get an estimate of $\rvm^*$ and evaluate the regret~\eqref{eq:regret} associated with some (past) treatment assignment in contexts $\rvC^*$.

\caption{Data-efficient regret evaluation}
\label{algo:method}
\end{algorithm}

For completeness we formalise the claim made in Section~\ref{sec:training_stage} in Proposition \ref{prop:prop1}, and provide a proof, following standard arguments.
\begin{proposition}[InfoNCE lower bound] 
	\label{thm:tractable_bound}
For any critic function $U\!:\! \mathcal{Y}  \times  \R \!\rightarrow\! \R$ and number of contrastive samples $L \geq 1$ we have $\mathcal{L}(\rvA; \rvC, \rvC^*, L) \leq \mathcal{I}(\rvA; \rvC, \rvC^*)$, where 
\begin{equation}
\hspace{-10pt}
    \mathcal{L}(\rvA, U; \rvC, \rvC^*, L) = \E_{p(\psi)p(\rvy, \rvm^*_0 | \rvC, \rvC^*, \Do{\rvA}, \psi)p(\rvm^*_{1:L},\psi_{1:L}| \rvC^*)} 
    \left[ 
        \log \frac{  \exp(U(\rvy , \rvm^*_0))}{ \frac{1}{L+1} \sum_\ell \exp(U(\rvy , \rvm^*_\ell)) }
    \right] 
\end{equation}
The bound is tight for the optimal critic $U^*(\rvy , \rvm^*) = \log p(\rvy |\rvm^*, \Do{\rvA},  \rvC) + c(\rvy)$ as $L \rightarrow \infty$, where $c(\rvy)$ is an arbitrary function depending only on the outcomes $\rvy$.
\label{prop:prop1}
\end{proposition}

\begin{proof}
Let $U\!:\! \mathcal{Y}  \times  \R \!\rightarrow\! \R$ be  any function (critic) and introduce the shorthand
\begin{align}
    g(\rvy, \rvm^*_{0:L}) \coloneqq \frac{  \exp(U(\rvy , \rvm^*_0))}{ \frac{1}{L+1} \sum_\ell \exp(U(\rvy , \rvm^*_\ell)) }
\end{align}
Multiply the mutual information objective of~Equation~\eqref{eq:mi_objective} by $g(\rvy, \rvm^*_{0:L})/g(\rvy, \rvm^*_{0:L})$ to get

\begin{align}
    \mathcal{I}(\rvA; \rvC, \rvC^*)   &= \E_{\unconditionaljoint} 
    \left[ 
        \log \frac{\unconditionaljoint}{ p(\rvy|  \rvC, \Do{\rvA})p(\rvm^*|\rvC^*) } 
    \right] \\
    & = \E_{p(\psi)p(\rvy, \rvm^*_0 | \rvC, \rvC^*, \Do{\rvA},  \psi)} 
    \left[ 
        \log \frac{\unconditionaljointzero}{ p(\rvy|  \rvC, \Do{\rvA})p(\rvm^*_0|\rvC^*) } 
    \right] \\
    &=\E_{p(\psi)p(\rvy, \rvm^*_0 | \rvC, \rvC^*, \Do{\rvA},  \psi)p(\rvm^*_{1:L},\psi_{1:L}| \rvC^*)} 
    \left[ 
        \log \frac{  p(\rvy |\rvm^*_0, \Do{\rvA},  \rvC) }{ p(\rvy|\Do{\rvA},  \rvC) }
    \right] \\
    &=\E_{p(\psi)p(\rvy, \rvm^*_0 | \rvC, \rvC^*, \Do{\rvA},  \psi)p(\rvm^*_{1:L},\psi_{1:L}| \rvC^*)} 
    \left[ 
        \log \frac{  p(\rvy |\rvm^*_0, \Do{\rvA},  \rvC) g(\rvy, \rvm^*_{0:L}) }{ p(\rvy|\Do{\rvA},  \rvC) g(\rvy, \rvm^*_{0:L}) }
    \right] \\
    \intertext{which we can split into two expectations---one that does not contain the implicit likelihoods and another that is equal to $\mathcal{L}(\rvA, U; \rvC, \rvC^*, L)$}
    \begin{split}
    &= \E_{p(\psi)p(\rvy, \rvm^* | \rvC, \rvC^*, \Do{\rvA},  \psi)p(\rvm^*_{1:L},\psi_{1:L}| \rvC^*)} 
    \left[ 
        \log \frac{  p(\rvy |\rvm^*_0, \Do{\rvA},  \rvC) }{ p(\rvy|\Do{\rvA},  \rvC) g(\rvy, \rvm^*_{0:L}) }
    \right] \\ 
    & \quad + \E_{p(\psi)p(\rvy, \rvm^* | \rvC, \rvC^*, \Do{\rvA},  \psi)p(\rvm^*_{1:L},\psi_{1:L}| \rvC^*)} 
    \left[ 
        \log  g(\rvy, \rvm^*_{0:L}) 
    \right] \\
    &= \text{KL}(p_1\,||\,p_2) + \mathcal{L}(\rvA, U; \rvC, \rvC^*, L) \\
    &\geq \mathcal{L}(\rvA, U; \rvC, \rvC^*, L)
    \end{split} 
\end{align}
where
\begin{align}
\hspace{-20pt}
 \text{KL}(p_1\,||\,p_2) & =  \E_{p(\psi)p(\rvy, \rvm^*_0 | \rvC, \rvC^*, \Do{\rvA},  \psi)p(\psi_{1:L})p(\rvm^*_{1:L}|\psi_{1:L}, \rvC^*)} 
    \left[ 
        \log \frac{  p(\rvy |\rvm^*_0, \Do{\rvA},  \rvC) }{ p(\rvy|\Do{\rvA},  \rvC) g(\rvy, \rvm^*_{0:L}) }
    \right] \\
    & = \E
    \left[ 
        \log \frac{  \unconditionaljointzero p(\rvm^*_{1:L}| \rvC^*)}{p(\rvm^*_{0}|\rvC^*)  p(\rvy|\Do{\rvA},  \rvC) g(\rvy, \rvm^*_{0:L}) p(\rvm^*_{1:L}| \rvC^*)}
    \right] 
\end{align}
where the expectation is with respect to $p(\psi) p(\rvy | \rvC, \Do{\rvA}, \psi)p(\rvm^*_0|\psi, \rvC^*) p(\psi_{1:L}) p(\rvm^*_{1:L}|\psi_{1:L}, \rvC^*)$. This  is a valid KL divergence (since $p_2=p(\rvm^*_{0}|\rvC^*)  p(\rvy|\Do{\rvA},  \rvC) g(\rvy, \rvm^*_{0:L}) p(\rvm^*_{1:L}|\rvC^*)$ integrates to one by symmetry argument) and hence non-negative. 

Finally, substituting $U^*(\rvy , \rvm^*) = \log p(\rvy |\rvm^*, \Do{\rvA},  \rvC) + c(\rvy)$ in the definition of $\mathcal{L}$
\begin{align}
    \mathcal{L}(\rvA, U^*; \rvC, \rvC^*, L) &=  \E_{p(\psi)p(\rvy, \rvm^*_0 | \Do{\rvA},  \rvC, \rvC^*, \psi)p(\rvm^*_{1:L}, \psi_{1:L}| \rvC^*)} 
    \left[ 
        \log \frac{  p(\rvy |\rvm^*, \Do{\rvA},  \rvC)}{ \frac{1}{L+1} \sum_\ell p(\rvy |\rvm^*_\ell, \Do{\rvA},  \rvC) }
    \right] 
\end{align}
which is monotonically increasing in $L$ and tight in the limit as $L \rightarrow \infty$ since 
\begin{equation}
   \frac{1}{L+1} \sum_\ell p(\rvy |\rvm^*_\ell, \Do{\rvA},  \rvC) \rightarrow p(\rvy |\Do{\rvA},  \rvC)~a.s.
\end{equation}
\cite[See e.g.][Theorem~1 for more detail]{foster2020unified}.
\end{proof}

\paragraph{Discrete action space}
Given $K \geq 2$ possible treatments, we learn a (non-deterministic) policy $\pi$ with parameters $\bm{\alpha}$ 
\vspace{-5pt}
\begin{equation}
\vspace{-5pt}
    \pi_{d, i} = \frac{\exp( (\log \alpha_{d, i} + g_{d, i}) / \tau )}{\sum_{j=1}^K \exp( (\log \alpha_{d, j} + g_{d,j}) / \tau )}, \quad i=1,\dots, K, \quad g_{d, i}, \sim \text{Gumbel}(0, 1), \quad \tau > 0,
\end{equation}
where the parameter $\tau$ is called the temperature, which we anneal during optimisation---starting the optimisation at high temperatures, when gradients exhibit low variance, and gradually reducing it towards 0, a which stage gradients are high variance, but we (should) have already learnt good treatments.  
We optimise the parameters $\bm{\alpha}$ and those of the critic network $\phi$ jointly with SGA. Once the policy is trained, the optimal design for experiment $d$ in the batch is $\rva_d = \argmax(\pi_{d, 1}, \dots, \pi_{d, K})$.


\section{Additional Related Work}
\label{sec:app:relwork}
Our setting is related to contextual Bayesian optimisation; some of the more closely related methods include Profile Expected Improvement \citep[PEI,][]{ginsbourger2014bayesian}, Multi-task Thompson Sampling  \citep[MTS,][]{char2019offline} and conditional Bayesian optimization \citep[ConBO,][]{pearce2020practical}. All of these methods are, however, restricted to GPs and the criteria they use to choose optimal designs are not information-based.

Contextual bandits is another broad framework that our work is related to. An extensive line of research is focused on \emph{online linear} bandits and discrete actions chosen using (variations of) UCB, Thompson sampling or $\epsilon$-greedy strategy \citep{auer2002using, chu2011contextual, agrawal2013thompson, han2020sequential}. 
\citet{krause2011contextual} instead model the reward as a GP defined over the context-action space and develop CGP-UCB. More recently, \cite{zanette2021design} proposed designing a batch of experiments \emph{offline} to collect a good dataset from which to learn a policy. 
Although this is similar to our setting, our approach is not restricted to linear rewards and uses an information-theoretic design objective.

Similar variational EIG objectives have been used in implicit likelihood BED methods for \emph{parameter learning}, but not for contextual optimisation. 
Gradient-based methods for large batch experimentation include SG-BOED \citep{foster2020unified} and MINEBED \citep{kleinegesse2020minebed}, while the policy-based iDAD \citep{ivanova2021implicit} applies to batch and adaptive settings. 
Our ability to handle discrete designs is another important distinction of our method. 


\section{Experiments} \label{sec:appendix_experiments}
\subsection{Training details}
We implement all experiments in Pyro \citep{pyro}. All experiments baselines ran for 50K gradient steps, using a batch size of 2048. We used the Adam optimiser \citep{kingma2014adam} with initial learning rate 0.001 and exponential learning rate annealing with coefficient 0.96 applied every 1000 steps. We used a separable critic architecture \citep{poole2018variational} with simple MLP encoders with ReLU activations and 32 output units. 

For the discrete treatment example:  we added \emph{batch norm} to the critic architecture, which helped to stabilise the optimisation. We had one hidden layer of size 512.
Additionally, for the Gumbel--Softmax policy, we started with a temperature $\tau=2.0$ and \texttt{hard=False} constraint.
We applied temperature annealing every 10K steps with a factor 0.5. We switch to \texttt{hard=True} in the last 10K steps of training.

For the continuous treatment example: We used MLPs with hidden layers of sizes $[\text{design dimension}\times 2, 412, 256]$ and 32 output units. 

Note: In order to evaluate the EIG of various baselines, we train a critic network for each one of them with the same hyperparameters as above.

\subsection{Posterior inference details}

After completing the training stage of our method (Algorithm~\ref{algo:method}), we need to deploy the learnt optimal designs in the real world in order to obtain rewards $\rvy$. This experimental data is then used to fit a posterior $p(\psi|\mathcal{D})$. 

\begin{wrapfigure}[16]{r}{0.35\textwidth}
  \centering
  \includegraphics[width=1.0\textwidth]{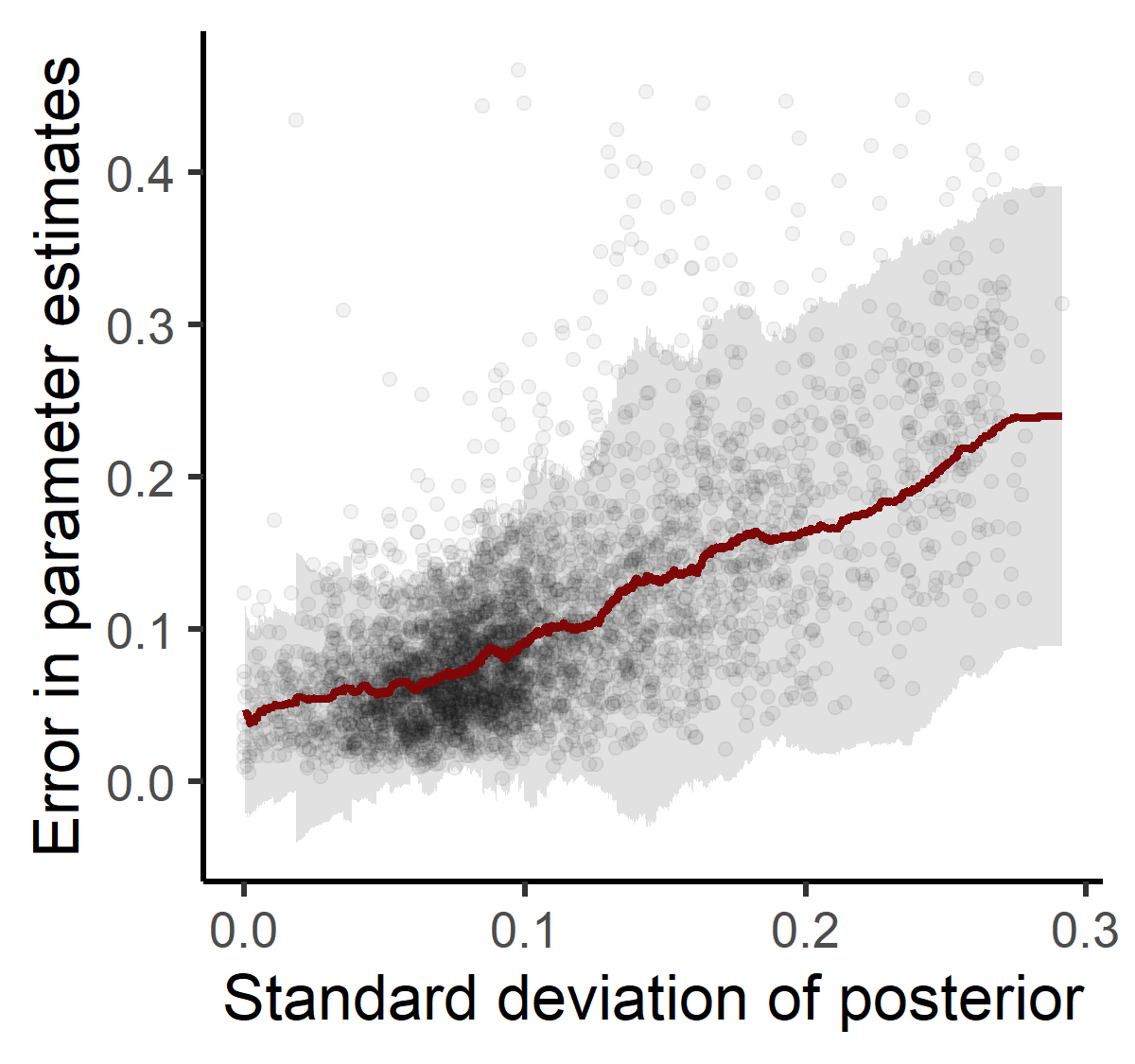}
    \caption{Posterior checks}
    \label{fig:calib}
\end{wrapfigure}
There are many ways to do the posterior inference and quality of the results will crucially depend on the accuracy of the fitted posteriors. In both of our examples and for all baselines we use use Pyro's self-normalised importance sampling (SNIS). Samples from this posterior are used for the evaluation metrics.

We validate the accuracy of estimated posteriors by running various sanity checks, including diagnostic plots such as Figure~\ref{fig:calib}, showing the standard deviation of our posterior mean estimate (measure of how uncertainty about the parameter) and $L_2$ error to the true parameter. The red line shows the rolling mean over 200 points of the latter, and the gray band---the 2 standard deviations. For this plot we used the continuous treatments example with $D=20$ experimental contexts.

\subsection{Evaluation metrics details}
As discussed in the main text, we evaluate how well we can estimate $\rvm^*$ by sampling a ground truth $\tilde{\psi}$ from the prior and obtaining a corresponding ground truth $\tilde{\rvm}^*$. 
We approximate the max-values $\rvm^*$ empirically using 2000 posterior samples of $\psi$ 
We similarly estimate $\psi$ using 2000 posterior samples. 
We define the optimal treatment under the posterior model to be average (with respect to that posterior) optimal treatment when treatments are continuous or $\text{UCB}_0$ when treatments are discrete. 
Finally the regret is computed as the average difference between the true max value (from the true environment and the true optimal treatment) and the one obtained by applying the estimated optimal treatment.  
We used 4000 (resp. 2000) true environment realisation for the continuous (resp. discrete) example. 

\subsection{Discrete treatments example}
\begin{table}[t]
\setlength{\tabcolsep}{6pt}
\renewcommand{\arraystretch}{0.98} 
\center
\small{
\begin{tabular}{lccccc}
Method             & EIG estimate      & MSE$(\rvm^*)$       & MSE$(\psi)$         &  Hit rate$(\rvA)$      & Regret     \\
\midrule
$\text{UCB}_{0.0}$ & 1.735 $\pm$ 0.005 & 2.541 $\pm$ 0.104 & 3.129 $\pm$ 0.054 & 0.513 $\pm$ 0.01 & 1.170 $\pm$ 0.036  \\
$\text{UCB}_{1.0}$ & 2.514 $\pm$ 0.006 & 1.003 $\pm$ 0.043 & 1.521 $\pm$ 0.021 & 0.496 $\pm$ 0.01 & 1.119 $\pm$ 0.035 \\
$\text{UCB}_{2.0}$ & 2.504 $\pm$ 0.006 & 0.965 $\pm$ 0.045 & 1.486 $\pm$ 0.021 & 0.497 $\pm$ 0.01 & 1.169 $\pm$ 0.037 \\
Random             & 3.573 $\pm$ 0.006 & 1.953 $\pm$ 0.070  & 1.347 $\pm$ 0.024 & 0.503 $\pm$ 0.01 & 1.150 $\pm$ 0.036  \\
\midrule
\textbf{Ours}    & \textbf{4.729} $\pm$\textbf{ 0.009} & \textbf{0.594} $\pm$ \textbf{0.025} & 1.326 $\pm$ 0.020  & 0.501 $\pm$ 0.01 & 1.152 $\pm$  0.035 \\
\bottomrule
\end{tabular}
}
\caption{Discrete treatments: evaluation metrics of 10D design. }
\label{tab:ab_oneseed}
\end{table}

\paragraph{Model}
We first give details about the toy model we consider in Figure~\ref{fig:exp_discrete_ab}. Each of the four treatments $\rva = 1, 2, 3, 4$ is a random function with two parameters $\psi_k = (\psi_{k,1}, \psi_{k,2})$ with the following Gaussian priors (parameterised by mean and covariance matrix): 
\begin{align}
    \psi_1 &~ \sim \mathcal{N} \left(\begin{pmatrix} 5.00 \\ 15.0  \end{pmatrix},  
    \begin{pmatrix}
     9.00 & 0 \\ 
     0 & 9.00
    \end{pmatrix} \right) &
    \psi_2 &~ \sim \mathcal{N} \left(\begin{pmatrix} 5.00 \\ 15.0  \end{pmatrix},  
    \begin{pmatrix}
     2.25 & 0 \\ 
     0 & 2.25
    \end{pmatrix} \right) \\
    \psi_3 &~ \sim \mathcal{N} \left(\begin{pmatrix} -2.0 \\ -1.0  \end{pmatrix},  
    \begin{pmatrix}
     1.21 & 0 \\ 
     0 & 1.21
    \end{pmatrix} \right) &
    \psi_4 &~ \sim \mathcal{N} \left(\begin{pmatrix} -7.0 \\ 3.0  \end{pmatrix},  
    \begin{pmatrix}
     1.21 & 0 \\ 
     0 & 1.21
    \end{pmatrix} \right)
\end{align}
and reward (outcome) likelihoods:
\begin{align}
    y | \rvc, \rva, \psi &\sim  \mathcal{N} \left( f(\rvc, \rva, \psi), 0.1 \right) \\
    f(\rvc, \rva, \psi) &= -\rvc^2 + \beta(\rva, \psi)\rvc + \gamma(\rva, \psi) \\
    \gamma &= (\psi_{\rva, 1} + \psi_{\rva, 2} + 18) / 2 \\
    \beta &= (\psi_{\rva, 2} - \gamma + 9) / 3
\end{align}
Intuition about the parameterisation: The first component of each $\psi_i$ defines the mean reward at context $\rvc=-3$, while the second one defines the mean reward at context $\rvc=3$. The reward is then the quadratic equation that passes through those points and leading coefficient equal to $-1$. 

\paragraph{Experimental and evaluation contexts} We use experimental and evaluation contexts of the same sizes. The experimental context, $\rvc$ is an equally spaced grid of size 10 between $-3$ and $-1$. We set the evaluation context $\rvc^* = -\rvc$. Figure~\ref{fig:exp_discrete_ab} in the main text visually illustrates this: the $x$-axis of the points in each plot are the experimental contests, while the dashed gray lines are the evaluation contexts.

\paragraph{Further results}
Table \ref{tab:ab_oneseed} shows all the evaluation metrics for the discrete treatment example from the main text. Our method achieves substantially higher EIG and lower MSE of estimating the max-rewards.  On all other metrics all methods perform similarly. This is to be expected, since Treatments 1 and 2 have exactly the same means and due to the way the model was parameterised (by the value of the quadratic at contexts 3 and -3), the probability of the optimal treatment being 1 or 2 is exactly $50\%$ (the hit rate all baselines achieve).
Note that $\text{UCB}_1$ and $\text{UCB}_2$ achieve statistically identical results, which is expected given they select the same designs.

\paragraph{Training stability}  We perform our method with the same hyperparameters but different training seeds and report the mean and standard error in Table~\ref{tab:ab_stability}.

\begin{table}[t]
\setlength{\tabcolsep}{6pt}
\renewcommand{\arraystretch}{0.98} 
\center
\small{
\begin{tabular}{lccccc}
Method             & EIG estimate      & MSE$(\rvm^*)$     & MSE$(\psi)$       & Hit rate$(\rvA)$  & Regret            \\
\midrule
$\text{UCB}_{0.0}$ & 1.740 $\pm$ 0.003  & 2.709 $\pm$ 0.058 & 3.197 $\pm$ 0.018 & 0.500 $\pm$ 0.005   & 1.150 $\pm$ 0.017  \\
$\text{UCB}_{1.0}$ & 2.508 $\pm$ 0.002 & 0.993 $\pm$ 0.016 & 1.529 $\pm$ 0.007 & 0.498 $\pm$ 0.004 & 1.140 $\pm$ 0.007  \\
$\text{UCB}_{2.0}$ & 2.505 $\pm$ 0.006 & 0.991 $\pm$ 0.023 & 1.518 $\pm$ 0.012 & 0.497 $\pm$ 0.003 & 1.145 $\pm$ 0.015 \\
Random             & 3.573 $\pm$ 0.333 & 2.369 $\pm$ 0.382 & 1.756 $\pm$ 0.269 & 0.502 $\pm$ 0.003 & 1.166 $\pm$ 0.008 \\
\midrule
Ours               & \textbf{4.769} $\pm$ \textbf{0.048} & \textbf{0.628} $\pm$ \textbf{0.025} &\textbf{ 1.369} $\pm$ \textbf{0.014} & 0.502 $\pm$ 0.005 & 1.160 $\pm$ 0.021 \\
\bottomrule
\end{tabular}
}
\caption{Discrete treatments example: 10D design, stability across training seeds}
\label{tab:ab_stability}
\end{table}

\subsection{Continuous treatment example}
\begin{figure}[t]
  \centering
  \includegraphics[width=1.0\textwidth]{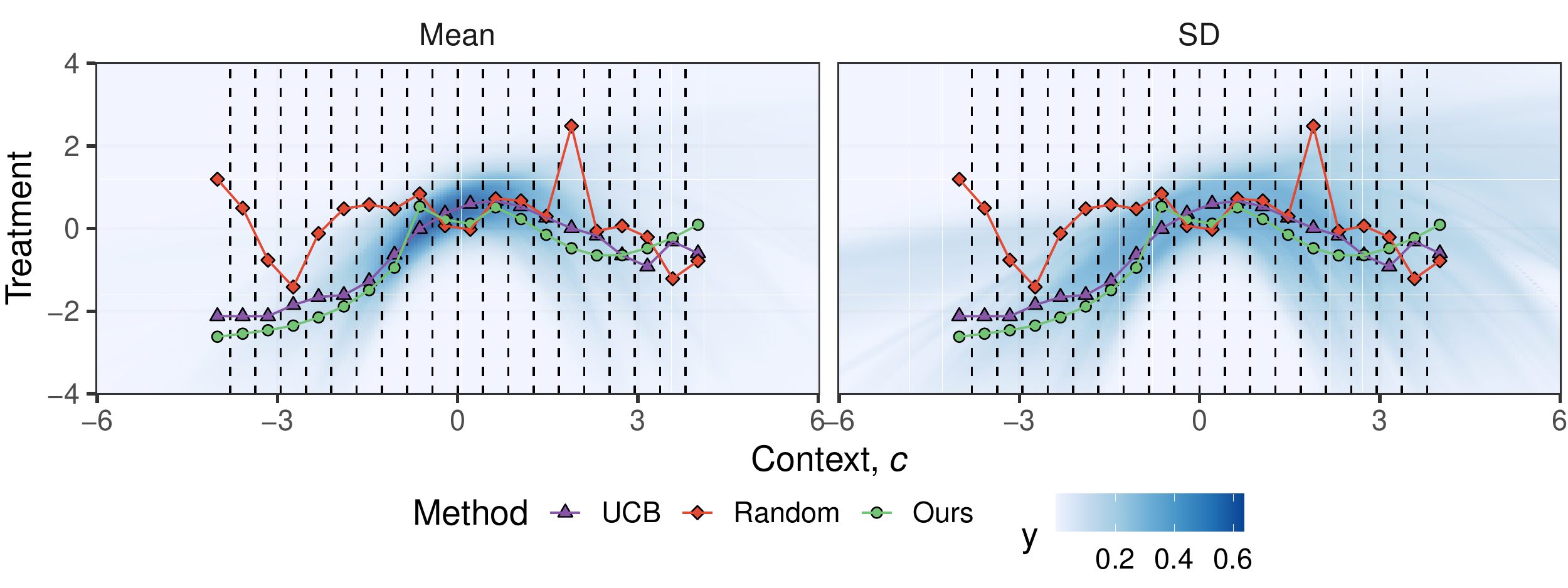}
    \caption{Continuous treatment example: 20D design to learn about 19 evaluation contexts (dashed gray lines). The random baseline samples treatments from $\mathcal{N}(0, 1)$.}
    \label{fig:cts_example_20d}
\end{figure}

\begin{table}[t]
\setlength{\tabcolsep}{3pt}
\renewcommand{\arraystretch}{1} 
\small{
\begin{tabular}{lccccc}
Method             & EIG estimate      & MSE$(\rvm^*)$       & MSE$(\psi)$         & MSE$(\rvA)$       & Regret            \\
\midrule
Random$_{0.2}$     & 5.548 $\pm$ 0.044 & 0.0037 $\pm$ 0.0002 & 0.0101 $\pm$ 0.0008 & 0.451 $\pm$ 0.033 & 0.083 $\pm$ 0.004 \\
Random$_{1.0}$     & 5.654 $\pm$ 0.128 & 0.0031 $\pm$ 0.0004 & 0.0065 $\pm$ 0.0008 & 0.343 $\pm$ 0.044 & 0.069 $\pm$ 0.006 \\
Random$_{2.0}$     & 5.118 $\pm$ 0.163 & 0.0045 $\pm$ 0.0003 & 0.0096 $\pm$ 0.0010  & 0.498 $\pm$ 0.032 & 0.086 $\pm$ 0.004 \\
$\text{UCB}_{0.0}$ & 5.768 $\pm$ 0.002 & 0.0066 $\pm$ 0.0002 & 0.0148 $\pm$ 0.0002 & 0.729 $\pm$ 0.022 & 0.082 $\pm$ 0.001 \\
$\text{UCB}_{1.0}$ & 5.892 $\pm$ 0.006 & 0.0031 $\pm$ 0.0001 & 0.0097 $\pm$ 0.0002 & 0.354 $\pm$ 0.013 & 0.068 $\pm$ 0.001 \\
$\text{UCB}_{2.0}$ & 5.797 $\pm$ 0.004 & 0.0030 $\pm$ 0.0001  & 0.0090 $\pm$ 0.0001  & 0.343 $\pm$ 0.011 & 0.071 $\pm$ 0.001 \\
\midrule
\textbf{Ours}               & \textbf{6.538} $\pm$ \textbf{0.008} & \textbf{0.0013} $\pm$ \textbf{0.0001} & \textbf{0.0038} $\pm$ \textbf{0.0001} & \textbf{0.131} $\pm$ \textbf{0.006} & \textbf{0.042} $\pm$ \textbf{0.0001}    \\
\bottomrule
\end{tabular}
}
\caption{Continuous treatments example: 40D design training stability. Mean and standard error are reported across 6 different training seeds.}
\label{tab:cts_40d_stability}
\end{table}

\paragraph{Model} For the continuous treatment example we use the following model:
\begin{align}
 \text{Prior: }  \quad  &\psi = (\psi_0, \psi_1, \psi_2, \psi_3 ), \quad  \psi_i \sim \text{Uniform}[0.1, 1.1] \; \text{iid} \\
 \text{Likelihood: } \quad &y  |\rvc, \rva, \psi \sim \mathcal{N} (f(\psi, \rva, \rvc), \sigma^2),
\end{align}
where
\begin{equation}
    f(\psi, \rva, \rvc) = \exp\left( -\frac{\big(a - g(\psi, \rvc)\big)^2}{h(\psi, \rvc)} \right) \quad
    g(\psi, \rvc) = \psi_0 + \psi_1 \rvc + \psi_2 \rvc^2 \quad
    h(\psi, c) = \psi_3
\end{equation}

\paragraph{Experimental and evaluation contexts} The experimental context, $\rvc$ is an equally spaced grid of size $D=40$ (or $20$ or $60$ in Further Results below) between $-3.5$ and and $3.5$. The evaluation context $\rvc^*$ is of size $D^* = D - 1$ and consists of the midpoints of the experimental context (see Figure~\ref{fig:cts_example_20d} for an illustration). 

\paragraph{Baselines} Since we have a continuous treatment, for the random baseline we consider sampling designs at random from $\mathcal{N}(0, 0.2)$, $\mathcal{N}(0, 1)$ or $\mathcal{N}(0, 2)$, which we denote by $\text{Random}_0.2$, $\text{Random}_1$ and $\text{Random}_2$, respectively.

\paragraph{Training stability} We perform our method with the same hyperparameters but different training seeds and report the mean and standard error in Table~\ref{tab:cts_40d_stability}.

\paragraph{Further results} We report the results of the same experiment, but  with a smaller and larger batch sizes of experimental and evaluation contexts. Table~\ref{tab:cts_20d} shows results for an experimental batch size of 20 contexts to learn about 19 evaluation contexts, while  Figure~\ref{fig:cts_example_20d} visually illustrates the model and the designs. Finally, Table~\ref{tab:cts_60d} shows results for an experimental batch size of 60 contexts to learn about 59 evaluation contexts.

\begin{table}[t]
\setlength{\tabcolsep}{3pt}
\renewcommand{\arraystretch}{1} 
\small{
\begin{tabular}{lccccc}
Method             & EIG estimate      & MSE$(\rvm^*)$       & MSE$(\psi)$         & MSE$(\rvA)$       & Regret    \\
\midrule
Random$_{0.2}$     & 4.262 $\pm$ 0.004 & 0.0086 $\pm$ 0.0003 & 0.0176 $\pm$ 0.0004   & 1.046 $\pm$ 0.041  & 0.120 $\pm$ 0.002 \\
Random$_{1.0}$     & 4.264 $\pm$ 0.004 & 0.0068 $\pm$ 0.0003 & 0.0158 $\pm$ 0.0004  & 0.799 $\pm$ 0.033& 0.114 $\pm$ 0.002 \\
Random$_{2.0}$     & 4.116 $\pm$ 0.003 & 0.0083 $\pm$ 0.0003 & 0.0198 $\pm$ 0.0005  & 1.002 $\pm$ 0.044& 0.127 $\pm$ 0.003 \\
$\text{UCB}_{0.0}$ & 5.093 $\pm$ 0.004 & 0.0074 $\pm$ 0.0004 & 0.0186 $\pm$ 0.0006  & 0.800 $\pm$ 0.047& 0.097 $\pm$ 0.002   \\
$\text{UCB}_{1.0}$ & 5.040 $\pm$ 0.004  & 0.0072 $\pm$ 0.0004 & 0.0180 $\pm$ 0.0006   & 0.764 $\pm$ 0.041& 0.097 $\pm$ 0.002 \\
$\text{UCB}_{2.0}$ & 5.038 $\pm$ 0.004 & 0.0048 $\pm$ 0.0003 & 0.0127 $\pm$ 0.0004  & 0.573 $\pm$ 0.033& 0.086 $\pm$ 0.002 \\
\midrule
\textbf{Ours}               & \textbf{5.642} $\pm$ \textbf{0.003} & \textbf{0.0034 } $\pm$ \textbf{0.0002} &\textbf{ 0.0073} $\pm$ \textbf{0.0003} & \textbf{0.065} $\pm$ \textbf{0.002} & \textbf{0.344} $\pm$ \textbf{0.027 } \\
\bottomrule
\end{tabular}
}
\caption{Continuous treatment example: 20D design to learn about 19 evaluation contexts.}
\label{tab:cts_20d}
\end{table}

\begin{table}[t]
\setlength{\tabcolsep}{3pt}
\renewcommand{\arraystretch}{0.9} 
\small{
\begin{tabular}{lccccc}
Method             & EIG estimate      & MSE$(\rvm^*)$       & MSE$(\psi)$         & MSE$(\rvA)$       & Regret    \\
\midrule
Random$_{0.2}$     & 6.033 $\pm$ 0.003 & 0.0026 $\pm$ 0.0001 & 0.0068 $\pm$ 0.0002 & 0.307 $\pm$ 0.019 & 0.068 $\pm$ 0.002 \\
Random$_{1.0}$     & 5.877 $\pm$ 0.004 & 0.0025 $\pm$ 0.0002 & 0.0058 $\pm$ 0.0002 & 0.310 $\pm$ 0.023  & 0.064 $\pm$ 0.002 \\
Random$_{2.0}$     & 6.153 $\pm$ 0.003 & 0.0022 $\pm$ 0.0002 & 0.0046 $\pm$ 0.0002 & 0.226 $\pm$ 0.019 & 0.055 $\pm$ 0.002 \\
$\text{UCB}_{0.0}$ & 6.106 $\pm$ 0.003 & 0.0056 $\pm$ 0.0004 & 0.0134 $\pm$ 0.0006 & 0.586 $\pm$ 0.045 & 0.077 $\pm$ 0.002 \\
$\text{UCB}_{1.0}$ & 6.200 $\pm$ 0.003   & 0.0027 $\pm$ 0.0003 & 0.0086 $\pm$ 0.0003 & 0.305 $\pm$ 0.028 & 0.063 $\pm$ 0.002 \\
$\text{UCB}_{2.0}$ & 6.234 $\pm$ 0.003 & 0.0024 $\pm$ 0.0002 & 0.0069 $\pm$ 0.0002 & 0.252 $\pm$ 0.024 & 0.064 $\pm$ 0.002 \\
\midrule
\textbf{Ours}     & \textbf{6.932} $\pm$ \textbf{0.003} & \textbf{0.0007} $\pm$ \textbf{0.0001} &\textbf{ 0.0027} $\pm$ \textbf{0.0001} & \textbf{0.069} $\pm$ \textbf{0.009} &\textbf{ 0.033} $\pm$ \textbf{0.001} \\
\bottomrule
\end{tabular}
}
\caption{Continuous treatment example: 60D design to learn about 59 evaluation contexts.}
\label{tab:cts_60d}
\end{table}

\end{document}